%% file: main.tex
\def\isarxiv{1} 

\ifdefined\isarxiv
\documentclass[11pt]{article}

\usepackage[numbers]{natbib}

\else

\documentclass[10pt,twocolumn,letterpaper]{article}

\fi

\ifdefined\isarxiv
\usepackage{amsmath}
\usepackage{amsthm}
\usepackage{amssymb}
\usepackage{algorithm}
\usepackage{subfig}
\usepackage{algpseudocode}
\usepackage{graphicx}
\usepackage{grffile}
\usepackage{wrapfig,epsfig}
\usepackage{url}
\usepackage{xcolor}
\usepackage{epstopdf}

\usepackage{bbm}
\usepackage{dsfont}

\else

\usepackage{amsmath}
\usepackage{amsthm}
\usepackage{amssymb}
\usepackage{algorithm}
\usepackage{subfig}
\usepackage{algpseudocode}
\usepackage{graphicx}
\usepackage{grffile}
\usepackage{wrapfig,epsfig}
\usepackage{epstopdf}

\usepackage{bbm}
\usepackage{dsfont}

\usepackage[review]{iccv}      

\fi

\allowdisplaybreaks

\ifdefined\isarxiv

\usepackage{tikz}
\usepackage{hyperref}  
\hypersetup{colorlinks=true,citecolor=blue,linkcolor=blue} 
\usetikzlibrary{arrows}
\usepackage[margin=1in]{geometry}

\else

\usepackage{microtype}
\usepackage{hyperref}
\definecolor{mydarkblue}{rgb}{0,0.08,0.45}
\hypersetup{colorlinks=true, citecolor=mydarkblue,linkcolor=mydarkblue}

 \definecolor{iccvblue}{rgb}{0.21,0.49,0.74}
\usepackage[pagebackref,breaklinks,colorlinks,allcolors=iccvblue]{hyperref}

\fi
 
\graphicspath{{./figs/}}

\theoremstyle{plain}
\newtheorem{theorem}{Theorem}[section]
\newtheorem{lemma}[theorem]{Lemma}
\newtheorem{definition}[theorem]{Definition}

\newtheorem{corollary}[theorem]{Corollary}

\newtheorem{remark}[theorem]{Remark}

\newcommand{\wh}{\widehat}
\newcommand{\wt}{\widetilde}

\newcommand{\R}{\mathbb{R}}

\DeclareMathOperator*{\E}{{\mathbb{E}}}
\DeclareMathOperator*{\var}{\mathrm{Var}}

\DeclareMathOperator{\poly}{poly}

\makeatletter
\newcommand*{\RN}[1]{\expandafter\@slowromancap\romannumeral #1@}
\makeatother

\usepackage{lineno}

\ifdefined\isarxiv

\else 

\title{
Time and Memory Trade-off of KV-Cache Compression in Tensor Transformer Decoding
}

\author{First Author\\
Institution1\\
Institution1 address\\
{\tt\small firstauthor@i1.org}
\and
Second Author\\
Institution2\\
First line of institution2 address\\
{\tt\small secondauthor@i2.org}
}

\fi

\begin{document}

\ifdefined\isarxiv

\date{}

\title{
Time and Memory Trade-off of KV-Cache Compression in Tensor Transformer Decoding
}

\author{
Yifang Chen\thanks{\texttt{
yifangc@uchicago.edu}. The University of Chicago.}
\and
Xiaoyu Li\thanks{\texttt{
xli216@stevens.edu}. Stevens Institute of Technology.}
\and
Yingyu Liang\thanks{\texttt{
yingyul@hku.hk}. The University of Hong Kong. \texttt{
yliang@cs.wisc.edu}. University of Wisconsin-Madison.} 
\and
Zhenmei Shi\thanks{\texttt{
zhmeishi@cs.wisc.edu}. University of Wisconsin-Madison.}
\and
Zhao Song\thanks{\texttt{ magic.linuxkde@gmail.com}. The Simons Institute for the Theory of Computing at UC Berkeley.}
\and
Yu Tian\thanks{\texttt{ kingyutian01@gmail.com}. Independent Researcher.}
}

\else

\fi

\ifdefined\isarxiv
\begin{titlepage}
  \maketitle
  \begin{abstract}
\input{0_abstract}

  \end{abstract}
  \thispagestyle{empty}
\end{titlepage}

{\hypersetup{linkcolor=black}
\tableofcontents
}
\newpage

\else
\maketitle
\begin{abstract}
\input{0_abstract}
\end{abstract}

\fi

\input{1_intro} 

\input{2_related_work}

\input{3_prelim}
\input{4_main_results}

\input{5_technical_results}

\input{6_low_regime}
\input{8_conclusion}





\newpage
\onecolumn
\appendix

\begin{center}
    \textbf{\LARGE Appendix }
\end{center}

\input{10_app_related_work}
\input{11_app_space_lowerbound}



\ifdefined\isarxiv
\bibliographystyle{alpha}
\bibliography{ref}

\else

{
    \small
    \bibliographystyle{ieeenat_fullname}
    \bibliography{ref}
}

\fi



\end{document}

%% file: 0_abstract.tex
The key-value (KV) cache in the tensor version of transformers presents a significant bottleneck during inference. While previous work analyzes the fundamental space complexity barriers in standard attention mechanisms [Haris and Onak, 2025], our work generalizes the space complexity barriers result to tensor attention version. Our theoretical contributions rely on a reduction from communication complexity and deduce the memory lower bound for tensor-structured attention mechanisms when $d = \Omega(\log n)$. Furthermore, we introduce two types of tensor attention cache and present a trade-off between time and memory for two scenarios.
Overall, our work provides a theoretical foundation for us to understand the time-memory tradeoff of KV-Cache compression in tensor attention decoding and offers more perspectives in developing more memory-efficient tensor attention Transformer architectures.

%% file: 1_intro.tex
\section{Introduction}

Transformers \cite{vsp+17} have emerged as the dominant architecture in modern deep learning systems, revolutionizing performance across natural language processing, computer vision, and multimodal tasks~\cite{bmr+20,dbk+21}. Many popular large language models including GPT-o3~\cite{gpto1}, Llama 3.3~\cite{llama3_arxiv,llama3_blog}, and Claude 3.7~\cite{claude3_pdf} have demonstrated unprecedented reasoning abilities. At the core of these architectures is the self-attention mechanism, which dynamically weights different parts of the input sequence, enabling models to capture long-range dependencies and contextual relationships~\cite{vsp+17,bmr+20}.

Autoregressive transformers represent a particularly influential variant that has dominated text generation tasks~\cite{kvpf20,sgs+24}. These models generate sequences one token at a time, with each prediction conditioned on all previously generated tokens~\cite{ydy+19}. 
During inference decoding, transformers utilize a masked self-attention mechanism that prevents the model from attending to future tokens, maintaining the causal structure necessary for proper text generation. While this approach has proven remarkably effective, it introduces unique computational challenges, particularly regarding memory usage and inference speed for long-context generation.

To resolve the significant memory challenges during inference, the concept of key-value (KV) cache is introduced in~\cite{pdc+23}. This KV cache, while essential for efficient inference by avoiding redundant computation, demands memory that scales linearly with sequence length. For large language models (LLMs) processing extensive contexts of tens or hundreds of thousands of tokens, this memory requirement still becomes prohibitively expensive.

Since the current KV cache grows linearly with sequence length, a range of work has been devoted to addressing this challenge~\cite{cdw+21,lwd+23,ldl+23,zsz+23,jlzm24,zdh24,xtc+24,zhmk24,llp+25} by compressing the KV cache and aiming for sublinear space complexity. However, these methods often rely on certain assumptions regarding the structure of key-value embeddings~\cite{ho25}. For example, the sliding window attention mechanism~\cite{dyl+24} only preserves key-value embeddings within a narrow range of preceding tokens and ignores embeddings lying outside this small window. Minicache~\cite{llp+25} proposes a method to merge two caches into a single one.

A recent study~\cite{ho25} has illuminated fundamental space complexity barriers of $\Omega(nd)$ for the KV cache compression of standard attention-based transformer. \cite{sht23,as24_iclr,lls+24_rope_tensor_tc0,lssz24_tat,zly+25} introduce a novel tensor attention-based transformer that can efficiently capture high-order correlations. Thus, it is natural to ask a question:

\begin{center}
   {\it What is the space complexity and running time for the KV cache in tensor-attention transformer decoding? }
\end{center}

In this work, we provide a theoretical foundation for understanding the compression, expressivity trade-off in tensor attention mechanisms, providing new perspectives into developing memory-efficient Transformer architectures. In particular, we introduce two kinds of KV cache for the tensor-attention transformer. We show that two-cache-matrix (Definition~\ref{def:attn_two_cache}) require less computation time, while four-cache-matrix (Definition~\ref{def:attn_four_cache}) require less memory.  

Our contribution can be described as follows:
\begin{itemize}
    \item We introduce two formulations of KV Cache for tensor Attention-based Autoregressive Transformers (see Definition~\ref{def:attn_four_cache} and~\ref{def:attn_two_cache}).
    \item We derive the running time and memory trade-off between two cache scenarios under $d = \Omega(\log n)$.
    The four-cache-matrix case requires a space lower bound of $\Omega(nd)$ and the two-cache-matrix case requires $\Omega(n^2 d)$ (Theorems~\ref{thm:tensor_lower_bound_four_matrices} and~\ref{thm:tensor_lower_bound_two_matrices}), while the two-cache-matrix case computation is faster by $\Omega(n^2 d)$ compared to the four-cache-matrix case (Theorem~\ref{thm:time_difference}).  
    \item We prove the optimality of \textsc{SubGen4Cache} and \textsc{SubGen2Cache} in the low dimensional regime $d = o(\log n)$ (see Theorem~\ref{thm:subgen_optimal_four} and~\ref{thm:subgen_optimal_two}).
\end{itemize}

Our derived theoretical bounds illuminate the fundamental space limitations of tensor attention-based Transformers and reveal an inherent trade-off between time and memory.

\paragraph{Roadmap.}
In Section~\ref{sec:related_work}, we review relevant literature related to our study.
In Section~\ref{sec:preli}, we provide the necessary background.
In Section~\ref{sec:main_res}, we present our main theoretical lower bound results for memory usage when $d = \Omega(\log n)$.
In Section~\ref{sec:tech_res}, we introduce key algorithms and mathematical foundations enabling efficient memory usage.
In Section~\ref{sec:low_dim}, we explore the space complexity of tensor attention in the low-dimensional regime where $d = o(\log n)$.
In Section~\ref{sec:conclusion}, we conclude the paper.

%% file: 2_related_work.tex
\section{Related Work}\label{sec:related_work}

\paragraph{Tensor Computation for High-order Representation.}
Tensors offer a superior framework over matrices for modeling multi-dimensional relationships inherent in data. The process of deriving low-rank factorizations or approximations of tensors is vital across a multitude of computer science domains, including natural language processing~\cite{cyym14, lzbj14, lzm+15, bnr+15}, computer vision~\cite{adtl09, ply10, lfc+16, clz17}, computer graphics~\cite{vt04, wws+05, v09}, security~\cite{acky05, acy06, kb06}, and data mining~\cite{kabo10, rst10, ks08, m11}. Moreover, tensors play an essential role in numerous machine learning applications~\cite{pblj15, jo14, hk13, alb13, zsj+17, ysst19, ssl+22} as well as in several other diverse fields~\cite{rtp16, cmd+15, zczj16, ycs16, rnss18}.

\paragraph{Transformers and Autoregressive Models in Computer Vision.}
Recent advancements in computer vision have been driven by the integration of Transformers and autoregressive models, significantly enhancing both discriminative and generative tasks. Initially designed for natural language processing, Transformers have been adapted for vision, with Vision Transformer (ViT) pioneering the treatment of images as patch sequences~\cite{dbk+21}. Subsequent works, such as Swin Transformer, leverage hierarchical designs and shifted-window self-attention to improve efficiency~\cite{llc+21}, while DETR reframes object detection as a set prediction task~\cite{cms+20}. In multimodal learning, CLIP employs natural language supervision to learn transferable visual representations, enabling zero-shot classification~\cite{rkh+21}. Similarly, DeiT and BEiT enhance training efficiency through knowledge distillation and self-supervised pre-training, respectively~\cite{tcd+21,bdpw22}. On the generative front, Image Transformer and DALL-E extend Transformer architectures to autoregressive image synthesis~\cite{pvu+18,rbhp21}, while Visual Autoregressive Modeling (VAR) redefines autoregressive learning as next-scale prediction, achieving superior performance over traditional autoregressive and diffusion models~\cite{tjy+24}. Building on this, FlowAR and ARFlow integrate flow-based techniques with attention mechanisms to advance high-resolution generation~\cite{ryh+24,hzy+25}. These developments highlight the complementary strengths of Transformers and autoregressive models in visual understanding and generation.

\paragraph{Fast Attention Computation.}
The attention mechanism is widely noted for having a computational cost that scales quadratically with the context length, a limitation that grows increasingly significant for contemporary LLMs~\cite{gpt4_arxiv,gpto1,llama3_arxiv,llama3_blog,claude3_pdf}.
Nevertheless, polynomial kernel approximation methods~\cite{aa22} offer a viable solution by enabling low-rank approximations of the attention matrix. This approach can substantially accelerate computations, allowing a single attention layer to train and infer at speeds approaching linear time~\cite{as23, as25_nips}. Our work further generalizes this efficiency across multiple transformer layers for both training and inference. Moreover, these kernel-based techniques naturally extend to more sophisticated attention variants, such as tensor attention, maintaining near-linear complexity in both training and evaluation~\cite{as24_iclr, lssz24}. Beyond these methods, alternative theoretical approaches also exist, with the conv-basis technique in~\cite{lls+24_conv} presenting another strategy for expediting attention computations.

%% file: 3_prelim.tex
\section{Preliminary}\label{sec:preli}

In this section, we provide the necessary preliminaries.
In Section~\ref{sec:preli:notation}, we present our notation.
In Section~\ref{sec:preli:columnwise_kronecker}, we introduce the column-wise Kronecker product and the softmax function.
In Section~\ref{sec:preli:jl_random_projection}, we present the Johnson-Lindenstrauss (JL) random projection.
In Section~\ref{sec:preli:attn_four}, we formalize the attention mechanism and outline the KV cache compression problem.
In Section~\ref{sec:preli:index_problem}, we define the \textsc{Index} problem.
In Section~\ref{sec:preli:index_lower_bound}, we establish the communication complexity lower bound for the \textsc{Index} problem.
In Section~\ref{sec:preli:clusterability}, we introduce the concept of clusterability for key embeddings.

\subsection{Notations}\label{sec:preli:notation}
For $n \in \mathbb{Z}^+$, where $\mathbb{Z}^+$ denotes positive integers, we use $[n]$ to denote set $\{1,2,\cdots, n\}$. We use $\E[\cdot]$ to denote the expectation.
We use $\var[\cdot]$ to represent the variance. We use $\Pr[\cdot]$ to denote the probability. We use $\mathbbm{1}\{C\}$ to represent an indicator variable that takes the value $1$ when condition $C$ holds and $0$ otherwise. $B^\top$ is defined as $(B^\top)_{i,j}:=B_{j,i}$. 
We use $x_{i,j}$ to denote the $j$-th coordinate of $x_i \in \R^n$.
We use $\|x\|_p$ to denote the $\ell_p$ norm of a vector $x \in \R^n$, i.e., $\|x\|_1 := \sum_{i=1}^n |x_i|$, $\|x\|_2 := (\sum_{i=1}^n x_i^2)^{1/2}$.
Suppose there are two vectors $c,d$ of the same length, we denote the entry-wise multiplication using the notation $c \odot d$; that is, the $i$-th entry in that vector is $c_i d_i$. 
We use ${\bf 0}_d$ to represent the vector of all zeros in dimension d.
We use $|M|$ to denote the length of a sequence $M$.
For two vectors $x \in \R^n$ and $y \in \R^n$, we use $\langle x, y \rangle$ to denote the inner product between $x,y$, i.e., $\langle x, y \rangle = \sum_{i=1}^n x_i y_i$.

\subsection{Useful Definitions}\label{sec:preli:columnwise_kronecker}

We state a standard notation which is called column-wise kronecker product. Such notation has been used in previous \cite{as24_iclr} (see Section~2).
 
\begin{definition}[$\oslash$ Column-wise Kronecker Product]\label{def:tensor_oslash}
Given matrices $K_1 \in \R^{n_1 \times d}, K_2 \in \R^{n_2 \times d}$, we define matrix $K := K_1 \oslash K_2 \in \R^{n_1 n_2 \times d}$ as follows
\begin{align*}
    K_{i_1 + (i_2-1)n_1 , j } := (K_1)_{i_1,j} \cdot (K_2)_{i_2,j}, \\~~~\forall i_1 \in [n_1], i_2 \in [n_2], j \in [d].
\end{align*}
\end{definition}

This operation constructs $K$ by combining elements of $K_1$ and $K_2$ in a column-wise manner, yielding a matrix with $n_1 n_2$ rows while preserving the column dimension $d$. It provides a compact and computationally efficient framework for subsequent tensor-based operations.

Next, we introduce the Softmax function, which is a fundamental unit in transformer architecture~\cite{vsp+17}.
\begin{definition}[Softmax]\label{def:softmax}
     The function $\sigma_{i^2}: \R^{i^2} \to \R^{i^2}$ is the softmax function, defined for any vector $z \in \R^{i^2}$ as:
    \begin{align*}
        \sigma_{i^2}(z)_j = \frac{\exp(z_j)}{\sum_{k=1}^{i^2} \exp(z_k)} ~~ \forall j \in [i^2]
    \end{align*}
\end{definition}

\subsection{Johnson-Linderstrauss (JL) Random Projections}\label{sec:preli:jl_random_projection}

In this section, we introduce a fundamental tool from random projection theory, the Johnson-Lindenstrauss (JL) Random Projection, which enables dimensionality reduction while approximately preserving inner products. We begin by defining the JL-transform \cite{jl84} and its properties following the framework established in~\cite{w14_book}.

\begin{definition}[JL-tranform \cite{jl84}, see Definition 3 in~\cite{w14_book} as an example]\label{def:jlt}
For a random matrix $S \in \R^{k \times n}$ forms a Johnson-Lindenstrauss transform with parameters $\epsilon, \delta, f$ or $\mathrm{JLT}(\epsilon, \delta, f)$ for short, if with probability at least $1 - \delta$, then for any $f$-element subset $V \subset \R^n$, for all $v,v' \in V$ it holds that
\begin{align*}
    | \langle Sv, Sv' \rangle - \langle v,v'\rangle | \leq \epsilon \|v \|_2 \| v'\|_2
\end{align*}
\end{definition}

The next lemma establishes when a random Gaussian matrix satisfies the JL-transform property.

\begin{lemma}[Theorem 4 in~\cite{w14_book}]\label{lem:jlt_gaussian}
If the following conditions hold:
\begin{itemize}
    \item Let $\epsilon, \delta \in (0,0.1)$.
    \item Let $S = \frac{1}{\sqrt{k}}R \in \R^{k\times n}$ where the entries $R_{i,j}$ of $R$ are independent standard normal random variables.
\end{itemize}
Then if $k = \Omega(\epsilon^{-2} \log (f/\delta))$, we have:
\begin{itemize}
    \item $S$ is a $\mathrm{JLT(\epsilon, \delta, f)}$.
\end{itemize}
\end{lemma}

By simple algebra, it is obvious that Lemma~\ref{lem:jlt_gaussian} will imply the following Lemma. For more examples of the following Lemma, we refer the readers to see Theorem 13 in \cite{w14_book}, Lemma A.2 in \cite{sxz22}, and Lemma 22 in \cite{ho25}.
\begin{lemma}[Johnson-Lindenstrauss (JL) Random Projections, an immediate application of Lemma~\ref{lem:jlt_gaussian} , \cite{w14_book}]\label{lem:inner-product-jl-for-all}
If the following conditions hold:
\begin{itemize}
    \item Let $\epsilon \in (0,0.1)$.
    \item Let $V = \{v_1,...,v_n\} \subset \R^n$ such that $\|v_i\|_2 \leq 1$ for all $i \in [n]$.
    \item Let $S = \frac{1}{\sqrt{d}}R \in \R^{d\times n}$ where the entries $R_{i,j}$ of $R$ are independent standard normal random variables.
\end{itemize}
Then if $d = \Omega(\epsilon^{-2}\log(n))$, we have:
\begin{itemize}
    \item With probability at least $1-1/\poly(n)$, for all $v_i, v_j \in V$, it holds that:
        \begin{align*}
        |v_i^\top v_j - (Sv_i)^\top (Sv_j)| \leq \epsilon.
        \end{align*}
\end{itemize}
\end{lemma}
For the completeness of the paper, we still provide a proof.
\begin{proof}
    Setting parameters $k = d$, $f = n$ and $\delta = \frac{1}{n}$ in Lemma~\ref{lem:jlt_gaussian}, we have if $d = \Omega(\epsilon^{-2}\log(n))$, then with probability $1-1/\poly(n)$, for all $v_i, v_j \in V$, we have
    \begin{align*}
        |v_i^\top v_j - (Sv_i)^\top (Sv_j)| \leq &~ \epsilon \|v_i\|_2 \|v_j\|_2 \\
        \leq &~ \epsilon
    \end{align*}
    where the first step follows from Definition~\ref{def:jlt}, and the second step follows from that $\|v_i\|_2 \leq 1$ for all $i \in [n]$.
\end{proof}

\subsection{Attention Mechanism with KV Cache}\label{sec:preli:attn_four}

We restate a definition from previous work \cite{as24_iclr}.

\begin{definition}[Tensor Attention, a variation of \cite{as24_iclr}]\label{def:tensor_attention}
    Let $d, i \in \R$ where $d$ is the embedding dimension and $i$ is the current sequence length. Let $\odot: \R^d \times \R^d \to \R^d$ denotes the Hadamard product and $\sigma_{i^2}$ is the Softmax function from Definition~\ref{def:softmax}.  Given a query vector $q_i \in \R^d$ and four matrices $K_{1,i}, K_{2,i}, V_{1,i}, V_{2,i} \in \R^{i \times d}$
    \begin{align*}
    K_{1,i} = \begin{bmatrix}
        k_{1,1}^\top\\
        k_{1,2}^\top\\
        \vdots\\
        k_{1,i}^\top
    \end{bmatrix},~
    K_{2,i} = \begin{bmatrix}
        k_{2,1}^\top\\
        k_{2,2}^\top\\
        \vdots\\
        k_{2,i}^\top
    \end{bmatrix},~V_{1,i} = \begin{bmatrix}
        v_{1,1}^\top\\
        v_{1,2}^\top\\
        \vdots\\
        v_{1,i}^\top
    \end{bmatrix}
    ,~V_{2,i} = \begin{bmatrix}
        v_{2,1}^\top\\
        v_{2,2}^\top\\
        \vdots\\
        v_{2,i}^\top
    \end{bmatrix}
    \end{align*}
    are four $i\times d$ matrices containing the historical key and value vectors.
    
    Then we are able to define the tensor attention function as:
    \begin{align*}
        \mathsf{Attn} &(q_i,K_{1,i},K_{2,i},V_{1,i},V_{2,i}) 
        := \underbrace{(V_{1,i} \oslash V_{2,i})^\top}_{d \times i^2} \cdot \sigma_{i^2}(\underbrace{(K_{1,i} \oslash K_{2,i})}_{i^2 \times d}\cdot \underbrace{q_i}_{d \times 1})  \in \R^{d},
    \end{align*}
     and $\oslash: \R^{i \times d} \times \R^{i \times d} \to \R^{i^2 \times d}$ denotes the column-wise Kronecker product from Definition~\ref{def:tensor_oslash}, which for any row indices $a, b \in [i]$ with corresponding flattened index $\ell = (a-1)i + b \in [i^2]$, gives:
    \begin{align*}
        (K_{1,i} \oslash K_{2,i})_{\ell} &= k_{1,a} \odot k_{2,b}\\
        (V_{1,i} \oslash V_{2,i})_{\ell} &= v_{1,a} \odot v_{2,b}
    \end{align*}
\end{definition}

Then we are able to present the definition of key-value cache in tensor attention.
\begin{definition}\label{def:attn_four_cache}
For a list of indices $1, \cdots, i$, we say matrices $K_{1,i},K_{2,i},V_{1,i}$ and $V_{2,i}$ (be defined as Definition~\ref{def:tensor_attention}) are collectively called the  key-value (KV) cache. Specifically, $K_{1,i}, K_{2,i}$ store key representations and $V_{1,i}, V_{2,i}$ store value representations, facilitating efficient tensorized attention computations.

\end{definition}

Next, we define the attention mechanism that consolidates the four KV cache matrices into two, leveraging the column-wise Kronecker product.

\begin{definition}[Tensor Attention with KV Cache, Two Cache Matrices Case]\label{def:attn_two_cache}
    Let $d, i \in \R$ where $d$ is the embedding dimension and $i$ is the current sequence length. Given a query vector $q_i \in \R^d$ and two matrices $\wt{K}_i, \wt{V}_i \in \R^{i^2 \times d}$, we define the tensor attention function as:
    \begin{align*}
        \mathsf{Attn}(q_i,\wt{K}_i,\wt{V}_i) := \wt{V}_i^\top \cdot \sigma_{i^2}(\wt{K}_i \cdot q_i) \in \R^d,
    \end{align*}
    and the two cache matrices $\wt{K}_i$ and $\wt{V}_i$ are derived from the column-wise Kronecker product (see Definition~\ref{def:tensor_oslash}) of the original key and value matrices:
    \begin{align*}
        \wt{K}_i &:= K_{1,i} \oslash K_{2,i} \in \R^{i^2 \times d}\\
        \wt{V}_i &:= V_{1,i} \oslash V_{2,i} \in \R^{i^2 \times d}
    \end{align*}
    where $K_{1,i}, K_{2,i}, V_{1,i}, V_{2,i} \in \R^{i \times d}$ are the four original cache matrices from Definition~\ref{def:attn_four_cache}, and $\sigma_{i^2}$ is the Softmax function from Definition~\ref{def:softmax}.
\end{definition}

This definition converts the original four KV cache matrices into two matrices, $\wt{K}_i$ and $\wt{V}_i$, each of size $i^2 \times d$, by precomputing the column-wise Kronecker products. This approach reflects the trade-off between preprocessing and space complexity, as discussed in Remark~\ref{rem:trade_off}.

\begin{remark}[Memory-Speed Trade-off Between Two and Four Cache Matrices]\label{rem:trade_off}
   The two formulations of tensor attention KV Caches present a trade-off between memory efficiency and computational speed during inference:
   
   \begin{itemize}
       \item {\bf Four Cache Matrices Case:} This approach stores $K_{1,i}, K_{2,i}, V_{1,i}, V_{2,i} \in \R^{i \times d}$, requiring $O(id)$ memory. However, during inference, it must compute the column-wise Kronecker products $K_{1,i} \oslash K_{2,i}$ and $V_{1,i} \oslash V_{2,i}$ on-the-fly, resulting in $O(i^2d)$ computational complexity for each attention operation.
       
       \item {\bf Two Cache Matrices Case:} This approach directly stores the precomputed Kronecker products $\wt{K}_i = K_{1,i} \oslash K_{2,i}$ and $\wt{V}_i = V_{1,i} \oslash V_{2,i}$, requiring $O(i^2d)$ memory. While this uses significantly more memory (quadratic in $i$ rather than linear), it enables faster inference as the expensive Kronecker product computation is eliminated.
   \end{itemize}
\end{remark}

\begin{figure}[!ht]
    \centering
    \includegraphics[width= 1.0\linewidth]{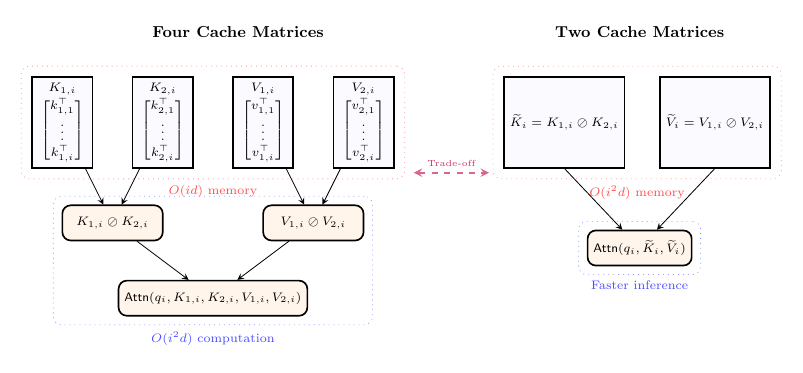}
    \caption{ Comparison of memory and computational requirements between four-cache and two-cache matrix formulations in tensor attention. The four-cache approach $(K_{1,i},K_{2,i},V_{1,i},V_{2,i})$ uses linear memory $O(id)$ but incurs $O(i^2d)$ computational cost for Kronecker products during inference. The two-cache approach $(\wt{K}_i, \wt{V}_i)$ pre-computes these products, requiring $O(i^2d)$ memory but enabling faster inference.}
    \label{fig:kv_caches}
\end{figure}

This illustrates a fundamental space-time trade-off in attention mechanisms: the four cache matrices formulation is memory-efficient but computationally expensive during inference, while the two cache matrices formulation uses more memory but achieves faster inference times (see Figure~\ref{fig:kv_caches}). Our lower bound proof in Theorem~\ref{thm:tensor_lower_bound_two_matrices} demonstrates that the memory requirement of $\Omega(n^2d)$ for the two cache matrices case is unavoidable, highlighting the inherent cost of optimizing for inference speed in this setting.

\subsection{The Index Communication Problem}\label{sec:preli:index_problem}

We use the same communication proof framework as \cite{ho25} to prove our lower bound. We state a well-known definition in the area of communication complexity.

\begin{definition}[The \textsc{Index} Problem \cite{kus97,ab09}]
In the \textsc{Index} problem, Alice holds a bit string $x \in \{0,1\}^n$, and Bob holds an index $i \in [n]$. Alice sends a single message (one-way) $M \in \{0,1\}^*$ to Bob, whose goal is to output $x_i$ with probability at least $2/3$.
\end{definition}

\subsection{Lower Bound on the Index Problem}\label{sec:preli:index_lower_bound}

We state a well-known tool in the area of communication complexity.

\begin{lemma}[Complexity of the \textsc{Index} Problem \cite{kus97,ab09}]\label{lem:index_lb}
The complexity of randomized one-way communication of \textsc{Index} is $\Omega(n)$.
\end{lemma}

\subsection{Clusterability of Key Embeddings}\label{sec:preli:clusterability}
We state a definition about Clusterability from recent work \cite{zhmk24,ho25}.
 
\begin{definition}[Clusterability Definition 1 on Page 5 in \cite{zhmk24}, see Definition 10 from~\cite{ho25} on Page 8 as an example]\label{def:clusterability}
For a positive integer $m$ and real-valued $\delta > 0$, a dataset of points $x_1,\hdots,x_n \in \R^d$ is said to be $(m,\delta)$-clusterable if there exists a partition of the dataset into $m$ clusters $C_1,\hdots,C_m$ such that for all $i \in [m]$ and $x,y \in C_i$ it holds that $\|x-y\|_2 \leq \delta$.
\end{definition}

%% file: 4_main_results.tex
\section{Space Complexity Barriers for \texorpdfstring{$d = \Omega(\log n)$}{}}\label{sec:main_res}

In this section, we present our lower bound results for memory usage when computing or approximating the tensor attention function when $d = \Omega(\log n)$.
In Section~\ref{sec:main_res:mem_lb_four}, we establish the memory lower bound for four cache matrices.
In Section~\ref{sec:main_res:mem_lb_two}, we establish the memory lower bound for two cache matrices.
In Section~\ref{sec:main_res:time_difference}, we present the time complexity difference between the two cache matrices case and the four cache matrices case.
In Section~\ref{sec:main_res:approx_alg_four}, we extend our memory lower bound to approximation algorithms for four cache matrices.
In Section~\ref{sec:main_res:approx_alg_two}, we extend our memory lower bound to approximation algorithms for two cache matrices.

\subsection{Memory Lower Bound for Four Cache Matrices}\label{sec:main_res:mem_lb_four}

Now we present our main theorem for the memory lower bound for computing the tensor attention function with four cache matrices, which is a variation of Theorem 6 in~\cite{ho25}. The difference is that we extend the proof from regular attention in~\cite{ho25} to tensor attention with four cache matrices from Definition~\ref{def:attn_four_cache} using similar proof structure in~\cite{ho25}.

\begin{theorem}[Memory Lower Bound for Four Cache Matrices, Informal Version of Theorem~\ref{thm:app_tensor_lower_bound_four_matrices}]\label{thm:tensor_lower_bound_four_matrices}
    If the following conditions hold:
    \begin{itemize}
        \item Let $C_u > 1$ be a universal constant.
        \item Let $d \geq C_u\log n$.
        \item Let the output of any algorithm produces $o_n \in \R^d$ where $o_n = \mathsf{Attn}(q_n,K_{1,n},K_{2,n},V_{1,n},V_{2,n})$.
    \end{itemize}
    Then we have:
    \begin{itemize}
        \item Any algorithm producing $o_n$ with probability at least $0.99$ use $\Omega(nd)$ bits of memory.
    \end{itemize}
\end{theorem}

\subsection{Memory Lower Bound for Two Cache Matrices}\label{sec:main_res:mem_lb_two}

Now we present our main theorem for the memory lower bound for computing the tensor attention function with two cache matrices, which is a variation of Theorem 6 in~\cite{ho25}. The difference is that we extend the proof from regular attention in~\cite{ho25} to tensor attention with two cache matrices from Definition~\ref{def:attn_two_cache} using similar proof structure in~\cite{ho25}.

\begin{theorem}[Memory Lower Bound for Two Cache Matrices, Informal Version of Theorem~\ref{thm:app_tensor_lower_bound_two_matrices}]\label{thm:tensor_lower_bound_two_matrices}
    If the following conditions hold:
    \begin{itemize}
        \item Let $C_u > 1$ be a universal constant.
        \item Let $d \geq C_u\log n$.
        \item Let the output of any algorithm produces $o_n \in \R^d$ where $o_n = \mathsf{Attn}(q_n,\wt{K}_n,\wt{V}_n)$, where $\wt{K}_n = K_{1,n} \oslash K_{2,n}$ and $\wt{V}_n = V_{1,n} \oslash V_{2,n}$.
    \end{itemize}
    Then we have:
    \begin{itemize}
        \item Any algorithm producing $o_n$ with probability at least $0.99$ must use at least $\Omega(n^2d)$ bits of memory.
    \end{itemize}
\end{theorem}

\subsection{Time Complexity Difference}\label{sec:main_res:time_difference}

In this section, we present the running complexity statements, where the two-cache-matrix case is running faster than the four-cache-matrix case during Tensor Transformer decoding. 

\begin{theorem}\label{thm:time_difference}
    Let the four-cache-matrix and two-cache-matrix computations be illustrated in Figure~\ref{fig:kv_caches}. We show that the two-cache-matrix case computation is faster by $\Omega(n^2 d)$ compared to the four-cache-matrix case computation.
\end{theorem}
\begin{proof}
    Following from Definition~\ref{def:attn_four_cache} and Definition~\ref{def:attn_two_cache}, this proof simply follows from the computation of $\mathsf{Attn}(q_n, K_{1,n}, K_{2,n}, V_{1,n}, V_{2,n})$ and $\mathsf{Attn}(q_n, \wt{K}_n, \wt{V}_n)$ and the time complexity of computation $\wt{K}_n = K_{1,n} \oslash K_{2,n}$ and $\wt{V}_n = V_{1,n} \oslash V_{2,n}$ is $\Omega(n^2 d)$. 
\end{proof}

\subsection{Lower Bound on Approximation Algorithms for Four Cache Matrices}\label{sec:main_res:approx_alg_four}

Having established memory lower bounds for exact computation of tensor attention, we now proceed to consider the case of approximation computation of the tensor attention.

\begin{corollary}[Approximation Algorithm Lower Bound for Four Cache Matrices, Informal Version of Corollary~\ref{thm:app_tensor_approx_lower_bound_four_matrices}]\label{thm:lower_bound_appx_four}
If the following conditions hold:
\begin{itemize}
    \item Let $Z_n := \mathsf{Attn}(q_n, K_{1,n}, K_{2,n}, V_{1,n}, V_{2,n})$ as in Definition~\ref{def:attn_four_cache}.
    \item Let $d = \Omega(\log n)$.
\end{itemize}
Then we have:
\begin{itemize}
    \item Any algorithm that can, with probability at least $0.99$, produce an output $\mathcal{O} \in \R^d$ that is a $(1\pm\eta)$-approximation of $Z_n$ for $\eta \in (0,1)$ must use at least $\Omega(nd)$ bits of memory.
\end{itemize}
\end{corollary}

\subsection{Lower Bound on Approximation Algorithms for Two Cache Matrices}\label{sec:main_res:approx_alg_two}

We now extend our investigation to approximation algorithms for the two cache matrices formulation. This case is particularly important for practical implementations that optimize for inference speed by precomputing the column-wise Kronecker products. The following theorem establishes that the quadratic memory requirement persists even for approximation algorithms:

\begin{corollary}[Approximation Algorithm Lower Bound for Two Cache Matrices, Informal Version of Corollary~\ref{thm:app_tensor_approx_lower_bound_two_matrices}]\label{thm:lower_bound_appx_two}
If the following conditions hold:
\begin{itemize}
    \item Let $Z_n := \mathsf{Attn}(q_n, \wt{K}_n, \wt{V}_n)$ as in Definition~\ref{def:attn_two_cache}.
    \item Let $d = \Omega(\log n)$.
\end{itemize}
Then we have:
\begin{itemize}
    \item Any algorithm that can, with probability at least $9/10$, produce an output $\mathcal{O} \in \R^d$ that is a $(1\pm\eta)$-approximation of $Z_n$ for $\eta \in (0,1)$ must use at least $\Omega(n^2d)$ bits of memory.
\end{itemize}
\end{corollary}

%% file: 5_technical_results.tex
\section{Main Algorithm for \textsc{SubGen} Clustering}\label{sec:tech_res}

This section introduces key algorithms and mathematical foundations that enable efficient memory usage in tensor attention mechanisms.
In Section~\ref{sec:tech_res:subgen_four}, we present the \textsc{SubGen4Cache} and \textsc{SubGen2Cache} algorithms for four and two cache matrices.
In Section~\ref{sec:tech_res:covering_number_upper_bound}, we provide an upper bound on the covering number of the unit sphere.
In Section~\ref{sec:tech_res:subgen_clusterability}, we establish the inherent clusterability of the \textsc{SubGen4Cache} and \textsc{SubGen2Cache} algorithms.

\subsection{The \textsc{SubGen4Cache} and \textsc{SubGen2Cache} Algorithms}\label{sec:tech_res:subgen_four}

Following the same framework as \cite{ho25}, we also use the tool \textsc{SubGen4Cache} algorithm which is from~\cite{zhmk24}
We present the \textsc{SubGen4Cache} (which can be viewed as a variation from previous work \cite{zhmk24,ho25}).

\begin{lemma}
[\textsc{SubGen4Cache} Algorithm for Four Cache Matrices, \cite{zhmk24,ho25}]\label{lem:subgen_four}
If the following conditions hold:
\begin{itemize}
    \item Let the input be a sequence of quintuples $\{(q_i, k_{1,i}, k_{2,i}, v_{1,i}, v_{2,i})\}_{i=1}^n$ in $(\R^d)^5$.
    \item Assume that the key embeddings $k_{1,i}$ and $k_{2,i}$ are $(m,\delta)$-clusterable.
    \item Let $\|q_i\|_2 \leq r$ for all $i \in [n]$.
\end{itemize}
Then there exists
\begin{itemize}
    \item An algorithm that uses $O(d\epsilon^{-2}\cdot (d+me^{2\delta r}\log n))$ bits of space and outputs an estimate $\wh{O}_n$ such that with probability at least $0.99$ it holds that:
    \begin{align*}
        & ~ \|\wh{O}_n - \mathsf{Attn}(q_n, K_{1,n}, K_{2,n}, V_{1,n}, V_{2,n})\|_2 \notag \\ \leq & ~ \epsilon \cdot \|\mathsf{Softmax}((K_{1,n} \oslash K_{2,n}) \cdot q_n)\|_2 \cdot \|V_{1,n} \oslash V_{2,n}\|_2
    \end{align*}
\end{itemize}
\end{lemma}

We extend Theorem 11 of \cite{ho25} and introduce the \textsc{SubGen2Cache} algorithm with two cache matrices, as defined in Definition~\ref{def:attn_two_cache}.

\begin{lemma}[\textsc{SubGen2Cache} Algorithm for Two Cache Matrices, \cite{zhmk24,ho25}]\label{lem:subgen_two}
If the following conditions hold:
\begin{itemize}
    \item Let the input be a sequence of triples $\{(q_i, \wt{k}_i, \wt{v}_i)\}_{i=1}^{n^2}$ in $(\R^d)^3$.
    \item Assume that the key embeddings $\wt{k}_i$ are $(m,\delta)$-clusterable.
    \item Let $\|q_i\|_2 \leq r$ for all $i \in [n]$.
\end{itemize}
Then there exists
\begin{itemize}
    \item An algorithm that uses $O(d\epsilon^{-2}\cdot (d+me^{2\delta r}\log n))$ bits of space and outputs an estimate $\wh{O}_n$ such that with probability at least $0.99$ it holds that:
    \begin{align*}
        & ~ \|\wh{O}_n - \mathsf{Attn}(q_n, \wt{K}_n, \wt{V}_n)\|_2 \notag \\ \leq & ~ \epsilon \cdot \|\mathsf{Softmax}(\wt{K}_n \cdot q_n)\|_2 \cdot \|\wt{V}_n\|_2
    \end{align*}
\end{itemize}
\end{lemma}

\subsection{Upper Bound on the Covering Number}\label{sec:tech_res:covering_number_upper_bound}

We state a standard tool from previous work~\cite{w17}.

\begin{lemma}[Covering Number of the Unit Sphere,~\cite{w17}]\label{lem:covering_number}
If the following conditions hold:
\begin{itemize}
    \item Let $B_2(1)$ be the unit sphere in $\R^d$ with respect to the $\ell_2$ norm.
\end{itemize}
Then we have:
\begin{itemize}
    \item The covering number of $B_2(1)$ with respect to the $\ell_2$ norm and radius $\delta \in (0,1)$ is bounded by:
    \begin{align*}
    N(B_2(1), \|\cdot\|_2, \delta) \leq ({3}/{\delta})^d
    \end{align*}
\end{itemize}
\end{lemma}

\subsection{Clusterability of \textsc{SubGen4Cache} and \textsc{SubGen2Cache} Algorithms}\label{sec:tech_res:subgen_clusterability}
 
We extend Lemma 14 of \cite{ho25} and present the following lemma showing the clusterability of \textsc{SubGen4Cache} and \textsc{SubGen2Cache} Algorithms.

\begin{lemma}[Clusterability of \textsc{SubGen4Cache} and \textsc{SubGen2Cache} Algorithms, Informal Version of Lemma~\ref{lem:app_subgen_clusterability}]\label{lem:subgen_clusterability}
If the following conditions hold:
\begin{itemize}
    \item Let $x_1,\hdots,x_n \in \R^d$ be a set of $n$ points in $d$-dimensional space such that $\|x_i\|_2 \leq 1$ for all $i \in [n]$.
\end{itemize}
Then:
\begin{itemize}
    \item This set is $(m,\delta)$ clusterable for $m = \lceil e^d \rceil$ and $\delta = e/3$.
\end{itemize}
\end{lemma}

%% file: 6_low_regime.tex
\section{Space Complexity in Low-Dimensional Regime where \texorpdfstring{$d = o(\log n)$}{}}\label{sec:low_dim}

Following the similar ideas in \cite{ho25}, we also provide several results for $d = o(\log n)$. In Section~\ref{sec:low_dim:lower_bound_low_dim_four}, we establish the memory lower bound for four cache matrices when $d = o(\log n)$.
In Section~\ref{sec:low_dim:lower_bound_low_dim_two}, we establish the memory lower bound for two cache matrices when $d = o(\log n)$.
In Section~\ref{sec:low_dim:subgen_optimal_four}, we prove the optimality of the \textsc{SubGen4Cache} algorithm. 
In Section~\ref{sec:low_dim:subgen_optimal_two}, we prove the optimality of the \textsc{SubGen2Cache} algorithm.

\subsection{Lower Bound in the Low-Dimensional Regime for Four Cache Matrices}\label{sec:low_dim:lower_bound_low_dim_four}

If $d = o(\log n)$, our preceding proof breaks down because the JL projection is not able to preserve the inner products of all pairs of key vectors with high probability. Our technique, however, still yields the following lower bound with four cache matrices: 

\begin{corollary}\label{cor:lower_bound_low_dim_four}
If the following conditions hold:
\begin{itemize}
    \item Let $Z_n := \mathsf{Attn}(q_n, K_{1,n}, K_{2,n}, V_{1,n}, V_{2,n})$ as in Definition~\ref{def:attn_four_cache}.
    \item Let $d \geq 2$ with $d = o(\log n)$.
    \item Let $p \geq 9/10$ denote the probability.
\end{itemize}
Then, there exists
\begin{itemize}
    \item  Any algorithm that can, with probability $p$, produce a $(1\pm\eta)$-approximation $\mathcal{O} \in \R^d$ of $Z_n$ for $\eta \in (0,1)$ must use at least $\Omega(e^d \cdot d)$ bits of memory.
\end{itemize}
 
\end{corollary}

\subsection{Lower Bound in the Low-Dimensional Regime for Two Cache Matrices}\label{sec:low_dim:lower_bound_low_dim_two}

Similarly, our technique still provides the following lower bound with two cache matrices when $d = o(\log n)$.
\begin{corollary}\label{cor:lower_bound_low_dim_two}
If the following conditions hold:
\begin{itemize}
    \item Let $Z_n := \mathsf{Attn}(q_n, \wt{K}_n, \wt{V}_n)$ as in Definition~\ref{def:attn_two_cache}.
    \item Let $d \geq 2$ with $d = o(\log n)$,
    \item Let $p \geq 9/10$ denote the probability.
\end{itemize}
Then, there exists
\begin{itemize}
    \item Any algorithm that can, with probability $p$, produce a $(1\pm\eta)$-approximation $\mathcal{O} \in \R^d$ of $Z_n$ for $\eta \in (0,1)$ must use at least $\Omega(e^d \cdot d)$ bits of memory.
\end{itemize}
\end{corollary}

\subsection{Optimality of \textsc{SubGen4Cache} in the Low-Dimensional Regime}\label{sec:low_dim:subgen_optimal_four}
The \textsc{SubGen4Cache} algorithm explores the inherent geometric structure of embeddings in low-dimensional spaces, where data naturally clusters into a limited number of groups. This key insight allows the algorithm to achieve exponentially better space complexity than would be possible in higher dimensions, while still maintaining strong approximation guarantees.

\begin{theorem}[Optimality of \textsc{SubGen4Cache}]\label{thm:subgen_optimal_four}
If the following conditions hold:
\begin{itemize}
    \item Let $2 \leq d = o(\log n)$.
    \item Define $\delta := e/3$.
    \item Define $r := O(\log\log n)$.
    \item Assume $\|k_{1,i}\|_2 \leq 1$ and $\|k_{2,i}\|_2 \leq 1$ for all $i \in [n]$, where $k_{1,i}$ and $k_{2,i}$ are key embeddings in $\mathsf{Attn}(q_n, K_{1,n}, K_{2,n}, V_{1,n}, V_{2,n})$.
\end{itemize}
Then:
\begin{itemize}
    \item The \textsc{SubGen4Cache} algorithm of Lemma~\ref{lem:subgen_four} approximates the attention function  with space complexity $\wt{O}(de^d)$.
\end{itemize}
\end{theorem}

\begin{proof}
By Lemma~\ref{lem:subgen_clusterability}, the key embeddings $k_{1,i}$ and $k_{2,i}$ are $(e^d, e/3)$-clusterable, and so we can apply Lemma~\ref{lem:subgen_four}, the space complexity of the \textsc{SubGen4Cache} algorithm is:
\begin{align*}
    O(d\epsilon^{-2}\cdot (d+e^{d+2\frac{e}{3} \log\log n}\log n)) = \wt{O}(de^d)
\end{align*}
The approximation guarantees are as in Lemma~\ref{lem:subgen_four}, though with some algebra, they could be extended to arbitrary absolute error approximation. Combining this with Corollary~\ref{cor:lower_bound_low_dim_four}, we see that \textsc{SubGen4Cache} is essentially optimal in the low-dimensional regime for the four cache matrices case.
\end{proof}

\subsection{Optimality of \textsc{SubGen2Cache} in the Low-Dimensional Regime}\label{sec:low_dim:subgen_optimal_two}
We now demonstrate that the \textsc{SubGen2Cache} algorithm is essentially optimal in the low-dimensional case.

\begin{theorem}[Optimality of \textsc{SubGen2Cache}]\label{thm:subgen_optimal_two}
If the following conditions hold:
\begin{itemize}
    \item Let $2 \leq d = o(\log n)$.
    \item Define $\delta := e/3$.
    \item Define $r := O(\log\log n)$.
    \item Assume $\|k_{1,i}\|_2 \leq 1$ and $\|k_{2,i}\|_2 \leq 1$ for all $i \in [n]$, where $k_{1,i}$ and $k_{2,i}$ are key embeddings underlying $\wt{K}_n = K_{1,n} \oslash K_{2,n}$ in $\mathsf{Attn}(q_n, \wt{K}_n, \wt{V}_n)$.
\end{itemize}
Then:
\begin{itemize}
    \item The \textsc{SubGen2Cache} algorithm of Lemma~\ref{lem:subgen_two} approximates the attention function  with space complexity $\wt{O}(de^d)$.
\end{itemize}
\end{theorem}

\begin{proof}
By Lemma~\ref{lem:subgen_clusterability}, the key embeddings $k_{1,i}$ and $k_{2,i}$ are $(e^d, e/3)$-clusterable, and so we can apply Lemma~\ref{lem:subgen_two}, the space complexity of the \textsc{SubGen2Cache} algorithm is: 
\begin{align*}
O(d\epsilon^{-2}\cdot (d+e^{d+2\frac{e}{3} \log\log n}\log n)) = \wt{O}(de^d).
\end{align*}

The approximation guarantees are as in Lemma~\ref{lem:subgen_two}, though with some algebra, they could be extended to arbitrary absolute error approximation. Combining this with Corollary~\ref{cor:lower_bound_low_dim_two}, we see that \textsc{SubGen2Cache} is essentially optimal in the low-dimensional regime for the two cache matrices case.
\end{proof}

%% file: 8_conclusion.tex
\section{Conclusion}\label{sec:conclusion}
In this paper, we have addressed the critical memory challenges posed by the key-value (KV) cache in the tensor version of Transformers, expanding the foundational space complexity barriers from standard attention to tensor attention. Through a reduction from communication complexity, we derived information-theoretic bounds that reveal how memory consumption scales under higher-order token correlations. Our results demonstrate that when $d = \Omega(\log n)$, there exists a running time and memory trade-off between two cache scenarios. The four-cache-matrix case requires a space lower bound of $\Omega(nd)$ and the two-cache-matrix case requires $\Omega(n^2 d)$, while the two-cache-matrix case computation is faster by $\Omega(n^2 d)$ compared to the four-cache-matrix case. By underscoring these theoretical and algorithmic insights, we aim to guide the development of more memory-efficient tensor attention Transformer architectures.

%% file: 10_app_related_work.tex
\section{More Related Work}\label{sec:more_related_work}

\paragraph{Large Language Models.}
Transformer architectures have quickly established themselves as the prevailing approach for addressing numerous natural language processing (NLP) challenges, owing to their expandability, adaptability, and capacity to identify intricate linguistic structures~\cite{vsp+17}. When these models expand to contain billions or trillions of parameters and undergo training on extensive, heterogeneous datasets, they are designated as large language models (LLMs) or foundation models~\cite{bha+21}. These LLMs are crafted to broadly apply across various downstream tasks, exhibiting remarkable versatility and effectiveness. Notable examples include BERT, which transformed contextual language comprehension~\cite{dclt19}, PaLM, which demonstrates excellence in multilingual and multitask functionalities~\cite{cnd+23}, Llama, which is enhanced for more efficient implementation in academic and commercial settings~\cite{tli+23}, along with widely-utilized systems such as ChatGPT and GPT-4, both of which have advanced the capabilities of conversational AI~\cite{chatgpt,gpt4}. These expansive models have shown exceptional transferability across an extensive range of downstream applications, spanning from machine translation and question-answering to summarization, text generation, and more sophisticated reasoning challenges~\cite{bce+23}. As LLMs continue to develop, various strategies have emerged to improve their adaptability and tailor them for particular use cases. A prevalent method involves incorporating adapter modules, which permit fine-tuning for new tasks without altering the entire model~\cite{eyp+22, zhz+23, ghz+23}.

Calibration approaches have also been introduced to refine model predictions for enhanced reliability across diverse inputs and environments~\cite{zwf+21, cpp+23}. To further enhance the effectiveness of LLMs in practical applications, multitask fine-tuning has gained significance, allowing models to be trained simultaneously on multiple related tasks, improving their performance across these domains~\cite{gfc+21b, vnr+23}. Complementing this are prompt-tuning techniques, where the input prompt is strategically designed to direct the model's response, enabling adaptation without extensive retraining~\cite{gfc+21b, lac+21}.

\paragraph{Theoretical Machine Learning.}
Our research is also informed by the following Machine Learning Theory work. Several studies examine neural network expressiveness through circuit complexity theory~\cite{lls+25_gnn,kll+25_var_tc0,lls+24_rope_tensor_tc0,cll+24_mamba,cll+24_rope,gkl+25_flowar_tc0}. Various works focus on optimizing algorithms that can enhance neural network training speed~\cite{llsz24,klsz24,dlms24,dswy22_coreset,haochen3,haochen4,dms23_spar,cll+25_deskreject,sy23,swyy23,lss+22,lsx+22,hst+22,hsw+22,hst+20,bsy23,dsy23,syyz23_weighted,gsy23_coin,gsy23_hyper,gsyz23,gswy23,syzz24,lsw+24,lsxy24,hsk+24,hlsl24,cll+25_universal_approximator,ccl+25,cgl+25}. Additional research analyzes neural networks through regression frameworks~\cite{cll+24_icl,gms23,lsz23_exp,gsx23,ssz23_tradeoff,css+23,syyz23_ellinf,syz23,swy23,syz23,lls+25_grok}. Some researchers employ reinforcement learning to enhance neural networks~\cite{haochen1,haochen2,yunfan1,yunfan2,yunfan3,yunfan4,lswy23}. Other works focus specifically on optimizing attention mechanisms~\cite{sxy23,lls+24_conv}.

\paragraph{Accelerating Attention Mechanisms.}
The attention mechanism encounters considerable computational hurdles due to its quadratic complexity relative to context length as sequence lengths increase in contemporary large language models~\cite{gpto1,llama3_blog,claude3_pdf}. To address this limitation, researchers have created polynomial kernel approximation methods~\cite{aa22} that utilize low-rank approximations to efficiently compute the attention matrix. These approaches significantly enhance computational efficiency, enabling attention layers to function with nearly linear time complexity during both training and inference~\cite{as23, as25_nips}. Such techniques have been successfully extended to more intricate attention variants, like tensor attention, while preserving almost linear computational efficiency~\cite{as24_iclr}.~\cite{kll+25} shows an almost linear time algorithm for enhancing VAR Transformer inference. Additional innovations in this area include RoPE-based attention mechanisms\cite{as24_rope,chl+24_rope} and privacy-preserving cross-attention techniques~\cite{lssz24_dp}. The conv-basis approach presented in~\cite{lls+24_conv} provides another method to improve attention computation speed, offering complementary solutions to this performance constraint. Researchers have also explored various pruning-based strategies to optimize attention mechanisms~\cite{lls+24_prune,cls+24,llss24_sparse,ssz+25_prune,ssz+25_dit,hyw+23,whl+24,xhh+24,ssz+25_prune}.

\paragraph{Gradient Approximation.}
Low-rank approximation represents a widely adopted approach for enhancing transformer training by decreasing computational complexity~\cite{lss+24,lssz24_tat,as25_nips,hwsl24,cls+24,lss+24_grad}. Expanding on the low-rank framework introduced in~\cite{as23}, which initially concentrated on forward attention computation,~\cite{as25_nips} extends this method to approximate attention gradients, effectively reducing the computational cost of gradient calculations. The research in~\cite{lss+24} further develops this low-rank gradient approximation for multi-layer transformers, demonstrating that backward computations in such architectures can achieve nearly linear time complexity. Additionally,~\cite{lssz24_tat} generalizes the approach of~\cite{as25_nips} to tensor-based attention models, utilizing forward computation results from~\cite{as24_iclr} to enable efficient training of tensorized attention mechanisms. Finally,~\cite{hwsl24} implements low-rank approximation techniques during the training of Diffusion Transformers (DiTs), illustrating the adaptability of these methods across various transformer-based architectures.

%% file: 11_app_space_lowerbound.tex
\section{Space Lower Bound}\label{sec:app_mem_lb}

In this section, we present our lower bound results for memory usage when computing or approximating the attention function. Section~\ref{sec:app_mem_lb:mem_lb_four} and~\ref{sec:app_mem_lb:mem_lb_two} establish memory lower bounds for computing the tensor attention function with four cache matrices and two cache matrices respectively.
Sections~\ref{sec:app_mem_lb:approx_alg_four} and~\ref{sec:app_mem_lb:approx_alg_two} extend these bounds to approximate algorithms.
Sections~\ref{sec:app_mem_lb:lower_bound_low_dim_four} and~\ref{sec:app_mem_lb:lower_bound_low_dim_two} analyze the lower bound in the low-dimensional regime ($d = o(\log n)$).
Section~\ref{sec:app_mem_lb:covering_number_def} defines the covering number.
Section~\ref{sec:app_mem_lb:subgen_clusterability} establishes the inherent clusterability of \textsc{SubGen4Cache} and \textsc{SubGen2Cache} algorithms.

\subsection{Memory Lower Bound for Four Cache Matrices Case}\label{sec:app_mem_lb:mem_lb_four}
We now establish a memory lower bound for computing the tensor attention function with four cache matrices. This result demonstrates that any algorithm achieving high-probability correctness requires linear memory in both the sequence length and embedding dimension.

\begin{theorem}[Memory Lower bound for Four Cache Matrices case, Formal Version of Theorem~\ref{thm:tensor_lower_bound_four_matrices}]\label{thm:app_tensor_lower_bound_four_matrices}
    If the following conditions hold:
    \begin{itemize}
        \item Let $C_u > 1$ be a universal constant.
        \item Let $d \geq C_u\log n$.
        \item Let the output of any algorithm produces $o_n \in \R^d$ where $o_n = \mathsf{Attn}(q_n,K_{1,n},K_{2,n},V_{1,n},V_{2,n})$.
    \end{itemize}
    Then we have:
    \begin{itemize}
        \item Any algorithm producing $o_n$ with probability at least $\frac{9}{10}$ use $\Omega(nd)$ bits of memory.
    \end{itemize}
\end{theorem}

\begin{proof}
    We construct a reduction from \textsc{Index}, in which Alice holds a bit string $x \in \{0,1\}^{n\times d}$ and Bob holds an index $(i,j) \in [n]\times [d]$. Alice sends a single message to Bob, whose goal is to output $x_{ij}$ with probability at least $2/3$. 
    
    We know by Lemma~\ref{lem:index_lb} that the one-way, randomized communication complexity of this problem is $\Omega(nd)$. Our goal is to design a protocol for \textsc{Index} by using a supposed algorithm $\mathcal{A}$ for calculating the attention function that uses $S$ bits of memory. Having that reduction in place, Alice simply communicates these $S$ bits of the algorithm's memory tape to Bob, allowing us to show that $S = \Omega(nd)$, and thus proving the theorem.

     \paragraph{Alice} Alice begins by inserting the following $n$ quintuples $\{(q_t,k_{1,t},k_{2,t},v_{1,t},v_{2,t})\}_{t=1}^n$ of vectors in $\R^d$ to $\mathcal{A}$:
    \begin{itemize}
         \item $q_1,\hdots,q_n$ are all the zero vector in $\R^d$.
         \item $k_{1,1},\hdots,k_{1,n} \in \R^d$ are calculated as $d$-dimensional projections of the orthonormal basis $e_1,\hdots,e_n$ of $\R^n$ that approximately preserve orthonormality. Specifically, using Lemma~\ref{lem:inner-product-jl-for-all}, we produce vectors such that with probability at least $1-\frac{1}{n}$ it holds for all $t\neq s$ that:
        \begin{align*}
        |k_{1,t}^\top k_{1,s}| \leq \epsilon
        \end{align*}
        and for all $t \in [n]$ that:
        \begin{align*}
        |k_{1,t}^\top k_{1,t} - 1| \leq \epsilon
        \end{align*}
        \item $k_{2,1},\hdots,k_{2,n} \in \R^d$ are all set to the all-ones vector $\mathbf{1}_d$.
        \item $v_{1,1},\hdots,v_{1,n} \in \R^d$ contain the rows of $x \in \{0,1\}^{n\times d}$, i.e., $v_{1,t} = x_t$ for all $t \in [n]$.
        \item $v_{2,1},\hdots,v_{2,n} \in \R^d$ are all set to the all-ones vector $\mathbf{1}_d$.
    \end{itemize}
    
    After inserting these $n$ quintuples into $\mathcal{A}$, Alice sends $\mathcal{A}$'s memory state of $S$ bits to Bob.

    \paragraph{Bob} Recall that Bob's input is an index $(i,j)$ into $x$. 
    
    Bob enters a single quintuple $(q,k_{1,n+1},k_{2,n+1},v_{1,n+1},v_{2,n+1})$ into $\mathcal{A}$, where:
    \begin{align*}
        q &:= C\cdot k_{1,i} \in \R^d \\
        k_{1,n+1} &:= \mathbf{0}_d \\
        k_{2,n+1} &:= \mathbf{0}_d \\
        v_{1,n+1} &:= \mathbf{0}_d \\
        v_{2,n+1} &:= \mathbf{0}_d
    \end{align*}
    where $C$ is a positive value to be determined later.
    
    Now, we claim that Bob can recover the value of $x_{ij}$ from the output $o_{n+1}$ that $\mathcal{A}$ produces. We analyze how the softmax distribution concentrates mass in the tensor attention calculation.
    \begin{align*}
         o_{n+1} 
         = & ~ \mathsf{Attn}(q_{n+1},K_{1,n+1},K_{2,n+1},V_{1,n+1},V_{2,n+1}) \\
         = & ~ (V_{1,n+1} \oslash V_{2,n+1})^\top \cdot \sigma_i((K_{1,n+1} \oslash K_{2,n+1})\cdot q_{n+1} )
    \end{align*}
    Let $s := (K_{1,n+1} \oslash K_{2,n+1}) \cdot q \in \R^{(n+1)^2}$. For each pair $(a, b) \in [n+1] \times [n+1]$ with flattened index $\ell = (a-1)(n+1) + b$, we have \begin{align}\label{eq:app_s_ell}
        s_\ell 
        = & ~ (k_{1,a} \odot k_{2,b})^\top q \notag \\
        = & ~ C \cdot (k_{1,a} \odot k_{2,b})^\top \cdot (k_{1,i})
    \end{align}
    where the first step follows from Definition~\ref{def:tensor_oslash}, and the second step follows from the definition of $q = C\cdot k_{1,i}$

    Then, we have
    \begin{align}\label{eq:app_c_epsilon}
        s_\ell 
        = & ~ s_{(a-1)(n+1) + b}  \notag \\
        = & ~ C \cdot (k_{1,a} \odot k_{2,b})^\top \cdot (k_{1,i}) \notag \\
        = & ~ C \cdot \sum_{t=1}^d (k_{1,a} \odot {\bf 1}_d)_t \cdot (k_{1,i})_t \notag \\
        = & ~ \begin{cases}
        C(1-\epsilon), & \mathrm{if~} a = i \\
        C\epsilon, & \mathrm{if~} a \neq i
        \end{cases}
    \end{align}
    where the first step follows from the definition of $\ell$, the second step follows Eq.~\eqref{eq:app_s_ell}, and the third step follows from $a, b\in [n]$ and $k_{2,b} = {\bf 1}_d$, and the fourth step follows from Lemma~\ref{lem:inner-product-jl-for-all}.
    
    Let $\xi := \sigma_{(n+1)^2}(s) \in \R^{(n+1)^2}$. Then our attention becomes 

    \begin{align}\label{eq:app_v_times_xi}
        (V_{1,n+1} \oslash V_{2,n+1})^\top \cdot \xi 
        = & ~ \sum_{\ell=1}^{(n+1)^2} \xi_{\ell} \cdot (V_{1,n+1} \oslash V_{2,n+1})_{\ell} \notag \\
        = & ~ \sum_{a=1}^{n+1} \sum_{b=1}^{n+1} \xi_{(a-1)(n+1)+b} \cdot (v_{1,a} \odot v_{2,b})
    \end{align}
    
    We will focus on the $j$-th component of the vector.

    Then we have,
    \begin{align}\label{eq:app_j_th_entry}
    ((V_{1,n+1} \oslash V_{2,n+1})^\top \cdot \xi)_j 
    = & ~ \sum_{a=1}^{n+1} \sum_{b=1}^{n+1} \xi_{(a-1)(n+1)+b} \cdot (v_{1,a} \odot v_{2,b})_j \notag \\
    = & ~ \sum_{a=1}^{n} \sum_{b=1}^{n} \xi_{(a-1)(n+1)+b} \cdot (x_{a,j} \cdot 1) \notag \\
    = & ~ \sum_{a=1}^{n} \sum_{b=1}^{n} \xi_{(a-1)(n+1)+b} \cdot x_{a,j} \notag \\
    = & ~ \sum_{a=1}^{n} x_{a,j} \cdot \sum_{b=1}^{n} \xi_{(a-1)(n+1)+b} \notag \\
    = & ~ \sum_{a=1}^{n} x_{a,j} \cdot (\sum_{b=1}^{n} \frac{e^{(k_{1,a} \odot k_{2,b})^\top \cdot q}}{\sum_{a'=1}^{n}\sum_{b'=1}^{n} e^{(k_{1,a'} \odot k_{2,b'})^\top \cdot q}})
\end{align}
where the first step follows from Eq.~\eqref{eq:app_v_times_xi}, the second step follows from $v_{1,a} = x_{a,*}$ for $a \in [n]$ and $v_{2,b} = {\bf 1}_d$ for $b \in [n]$, the third step follows from basic algebra, the fourths step follows from basic algebra, the fifth step follows from Definition~\ref{def:attn_four_cache}.

By Eq.~\eqref{eq:app_c_epsilon}, we have
\begin{align}\label{eq:app_xi_cases}
    \sum_{b=1}^{n} \xi_{(a-1)(n+1)+b} = 
    \begin{cases}
        \frac{n \cdot e^{C(1-\epsilon)}}{n \cdot e^{C(1-\epsilon)} + n(n-1) \cdot e^{C\epsilon}}, & \mathrm{if~} a = i \\
        \frac{n \cdot e^{C\epsilon}}{n \cdot e^{C(1-\epsilon)} + n(n-1) \cdot e^{C\epsilon}}, & \mathrm{if~} a \neq i
    \end{cases}
\end{align}

Thus, we have
\begin{align*}
    ((V_{1,n+1} \oslash V_{2,n+1})^\top \cdot \xi)_j 
    = & ~ x_{i,j} \cdot \sum_{b=1}^{n} \xi_{(i-1)(n+1)+b} + \sum_{a \in [n] \setminus \{i\}} x_{a,j} \cdot \sum_{b=1}^{n} \xi_{(a-1)(n+1)+b} \\
    = & ~ x_{i,j} \cdot \frac{n \cdot e^{C(1-\epsilon)}}{n \cdot e^{C(1-\epsilon)} + n(n-1) \cdot e^{C\epsilon}} + \sum_{a \in [n] \setminus \{i\}} x_{a,j} \cdot \frac{n \cdot e^{C\epsilon}}{n \cdot e^{C(1-\epsilon)} + n(n-1) \cdot e^{C\epsilon}}
\end{align*}
where the first step follows from Eq.~\eqref{eq:app_j_th_entry}, and the second step follows from Eq.~\eqref{eq:app_xi_cases}. 

Now we consider two cases, $x_{i,j} = 0$ and $x_{i,j} = 1$.

{\bf Case One} ${\bf x_{i,j} = 0}$:
\begin{align*}
    ((V_{1,n+1} \oslash V_{2,n+1})^\top \cdot \xi)_j 
    = & ~ 0 \cdot \frac{n \cdot e^{C(1-\epsilon)}}{n \cdot e^{C(1-\epsilon)} + n(n-1) \cdot e^{C\epsilon}} + \sum_{a \in [n] \setminus \{i\}} x_{a,j} \cdot \frac{n \cdot e^{C\epsilon}}{n \cdot e^{C(1-\epsilon)} + n(n-1) \cdot e^{C\epsilon}} \\
    = & ~ \sum_{a \in [n] \setminus \{i\}} x_{a,j} \cdot \frac{n \cdot e^{C\epsilon}}{n \cdot e^{C(1-\epsilon)} + n(n-1) \cdot e^{C\epsilon}} \\
    \leq & ~ n \cdot \frac{n \cdot e^{C\epsilon}}{n \cdot e^{C(1-\epsilon)} + n(n-1) \cdot e^{C\epsilon}} \\
    = & ~ \frac{n^2 \cdot e^{C\epsilon}}{n \cdot e^{C(1-\epsilon)} + n(n-1) \cdot e^{C\epsilon}} := \delta
\end{align*}
where the first step follows from basic algebra, the second step follows from basic algebra, the third step follows from $\sum_{a \in [n] \setminus \{i\}} x_{a,j} \leq (n-1)$, and the fourth step follows from basic algebra.

{\bf Case Two} ${\bf x_{i,j} = 1}$:
\begin{align*}
    ((V_{1,n+1} \oslash V_{2,n+1})^\top \cdot \xi)_j 
    = & ~ 1 \cdot \frac{n \cdot e^{C(1-\epsilon)}}{n \cdot e^{C(1-\epsilon)} + n(n-1) \cdot e^{C\epsilon}} + \sum_{a \in [n] \setminus \{i\}} x_{a,j} \cdot \frac{n \cdot e^{C\epsilon}}{n \cdot e^{C(1-\epsilon)} + n(n-1) \cdot e^{C\epsilon}} \\
    \geq & ~ \frac{n \cdot e^{C(1-\epsilon)}}{n \cdot e^{C(1-\epsilon)} + n(n-1) \cdot e^{C\epsilon}} := \Delta
\end{align*}
where the first step follows from basic algebra, the second step follows from the basic lower bound definition.

For Bob to distinguish two cases, we need to compare the upper bound of case one with the lower bound of case two. 

Since both $\delta$ and $\Delta$ have the same denominator, they both have an $n$ factor in the numerator. Thus, if the final goal is to show $\delta < \Delta$, it is sufficient to show 
\begin{align*}
    n \cdot e^{C\epsilon} < e^{C(1-\epsilon)}
\end{align*}
The above equation is equivalent to 
\begin{align*}
    n < e^{C(1-2\epsilon)}
\end{align*}
which is further equivalent to 
\begin{align}\label{eq:app_c_inequality}
     C &> \frac{\ln n}{1-2\epsilon}
\end{align}

When $n \geq 3$ and $\epsilon \in (0,0.1)$, we choose $C = 2 \ln n$, which we automatically satisfy Eq.~\eqref{eq:app_c_inequality}, allowing Bob to correctly determine the $j$-th component of the attention output.

Then we conclude the overall bits communicated is equal to the memory of $\mathcal{A}$, which is $\Omega(nd)$ bits.

Thus, we complete the proof.
\end{proof}

\subsection{Memory Lower Bound for Two Cache Matrices Case}\label{sec:app_mem_lb:mem_lb_two}

The goal of this section is to prove Theorem~\ref{thm:app_tensor_lower_bound_two_matrices}.

\begin{theorem}[Memory Lower Bound for Two Cache Matrices case, Formal Version of Theorem~\ref{thm:tensor_lower_bound_two_matrices}]\label{thm:app_tensor_lower_bound_two_matrices}
    If the following conditions hold:
    \begin{itemize}
        \item Let $C_u > 1$ be a universal constant.
        \item Let $d \geq C_u\log n$.
        \item Let the output of any algorithm produces $o_n \in \R^d$ where $o_n = \mathsf{Attn}(q_n,\wt{K}_n,\wt{V}_n)$, where $\wt{K}_n = K_{1,n} \oslash K_{2,n}$ and $\wt{V}_n = V_{1,n} \oslash V_{2,n}$.
    \end{itemize}
    Then we have:
    \begin{itemize}
        \item Any algorithm producing $o_n$ with probability at least $\frac{9}{10}$ must use at least $\Omega(n^2d)$ bits of memory.
    \end{itemize}
\end{theorem}

\begin{proof}
    We construct a reduction from the \textsc{Index} communication problem. In this context, Alice holds a bit string $x \in \{0,1\}^{n^2 \times d}$ (which can be viewed as an $n^2 \times d$ matrix), and Bob holds an index pair $(i,j) \in [n^2] \times [d]$. Alice sends a single message to Bob, whose goal is to output $x_{i,j}$ with probability at least $2/3$. The communication complexity of this problem is $\Omega(n^2d)$.

    Let $\mathcal{A}$ be a streaming algorithm for calculating the tensor attention function using $S$ bits of memory. We will show that $S = \Omega(n^2d)$ by constructing a protocol for the \textsc{Index} problem using $\mathcal{A}$.

    {\bf Alice:}
    Alice inserts the following $n^2$ triples $\{(q_t, \wt{k}_t, \wt{v}_t)\}_{t=1}^{n^2}$ into the streaming algorithm $\mathcal{A}$:

    \begin{itemize}
        \item $q_1, \hdots, q_{n^2}$ are all the zero vector in $\R^d$.
        \item $\wt{k}_1, \hdots, \wt{k}_{n^2} \in \R^d$ are calculated before the protocol starts using the JL lemma to produce vectors such that with probability at least $1 - \frac{1}{n^2}$ it holds for all $i \neq j$ that:
        \begin{align*}
        |\wt{k}_i^T \wt{k}_j| \leq \epsilon
        \end{align*}
        and for all $i \in [n^2]$ that:
        \begin{align*}
        |\wt{k}_i^T \wt{k}_i - 1| \leq \epsilon
        \end{align*}
        \item $\wt{v}_1, \hdots, \wt{v}_{n^2} \in \R^d$ contain the rows of $x \in \{0,1\}^{n^2 \times d}$. In other words:
        \begin{align*}
            \wt{V}_{n^2} := x
        \end{align*}
    \end{itemize}

    After inserting these $n^2$ triples into $\mathcal{A}$, Alice sends its memory state ($S$ bits) to Bob.

    {\bf Bob:} With the index $(i,j) \in [n^2] \times [d]$, Bob enters a single triple $(q, \wt{k}_{n^2+1}, \wt{v}_{n^2+1})$ into $\mathcal{A}$:
    \begin{align*}
        q &:= C \cdot \wt{k}_i \\
        \wt{k}_{n^2+1} &:= {\bf 0}_d \\
        \wt{v}_{n^2+1} &:= {\bf 0}_d
    \end{align*}
    where $C$ is a positive value to be determined later.

    Now, Bob can recover the value of $x_{i,j}$ from the output $o_{n^2+1}$ that $\mathcal{A}$ produces:
    \begin{align*}
        o_{n^2+1} = \mathsf{Attn}(q, \wt{K}_{n^2+1}, \wt{V}_{n^2+1}) = \wt{V}_{n^2+1}^\top \cdot \sigma_{n^2+1}(\wt{K}_{n^2+1} \cdot q)
    \end{align*}

    Starting with $s := \wt{K}_{n^2+1} \cdot q \in \R^{n^2+1}$, we know with probability at least $1 - \frac{1}{n^2}$ that:
    \begin{align*}
    s_\ell \leq C\epsilon \mathrm{~for~} \ell \neq i \mathrm{~and~} s_i \geq C(1-\epsilon)
    \end{align*}

    Let $\xi := \sigma_{n^2+1}(s) \in \R^{n^2+1}$. The softmax vector $\xi$ ``spikes'' at index $i$, which allows Bob to extract $x_{i,j}$ via the value matrix. The $j$-th entry is as follows with $\wt{v}_{n^2+1} = {\bf 0}_d$..
    \begin{align*}
        (\wt{V}_{n^2+1}^\top \cdot \xi)_j = \sum_{\ell=1}^{n^2+1} \xi_\ell \cdot (\widetilde{v}_\ell)_j = \sum_{\ell=1}^{n^2} \xi_\ell \cdot x_{\ell,j}
    \end{align*}

    Then, we have two cases: $x_{i,j} = 0$ and $x_{i,j} = 1$.

    {\bf Case One} ${x_{i,j} = 0}$:
    \begin{align*}
        \sum_{\ell=1}^{n^2} \xi_\ell \cdot x_{\ell,j} 
        = & ~ 0 \cdot \xi_i + \sum_{\ell \in [n^2] \setminus\{i\}} \xi_\ell \cdot x_{\ell,j} \\
        = & ~ \sum_{\ell \in [n^2] \setminus\{i\}} \xi_\ell \cdot x_{\ell,j} \\
        \leq & ~ \sum_{\ell \in [n^2] \setminus\{i\}} \xi_\ell 
    \end{align*}
    where the first step follows from substitution, the second step follows from basic algebra, and the third step follows from $x_{\ell,j} \leq 1$.

    Then, we can further show that 
    \begin{align*}
        \sum_{\ell \in [n^2] \setminus\{i\}} \xi_\ell = & ~ 1 - \xi_i \\
        = & ~ 1 - \frac{e^{s_i}}{\sum_{\ell=1}^{n^2} e^{s_\ell}} \\
        = & ~ \frac{\sum_{\ell=1}^{n^2} e^{s_\ell} - e^{s_i}}{\sum_{\ell=1}^{n^2} e^{s_\ell}} \\
        = & ~ \frac{\sum_{\ell \in [n^2] \setminus\{i\}} e^{s_\ell}}{\sum_{\ell=1}^{n^2} e^{s_\ell}}
    \end{align*}
    where the first step follows from basic algebra, the second step follows from the definition of $\xi_i$, the third step follows from basic algebra, and the fourth step follows from basic summation rules.
    
    Given that $s_\ell \leq C\epsilon$ for $\ell \neq i$ and $s_i \geq C(1-\epsilon)$, we can bound this as:
    \begin{align*}
        \frac{\sum_{\ell \in [n^2] \setminus\{i\}} e^{s_\ell}}{\sum_{\ell=1}^{n^2} e^{s_\ell}} 
        \leq \frac{n^2 \cdot e^{C\epsilon}}{e^{C(1-\epsilon)} + n^2 \cdot e^{C\epsilon}} := \delta
    \end{align*}

    {\bf Case Two} ${x_{i,j} = 1}$:
    By substituting $ x_{i,j} = 1$, we can have
    \begin{align*}
        \sum_{\ell=1}^{n^2} \xi_\ell \cdot x_{\ell,j} 
        = & ~ \xi_i + \sum_{\ell \in [n^2] \setminus\{i\}} \xi_\ell \cdot x_{\ell,j}
    \end{align*}
    
    Since $\xi_i = \frac{e^{s_i}}{\sum_{\ell=1}^{n^2} e^{s_\ell}}$, we can establish a lower bound:
    \begin{align*}
        \sum_{\ell=1}^{n^2} \xi_\ell \cdot x_{\ell,j} \geq & ~ \xi_i \\
        = & ~ \frac{e^{s_i}}{\sum_{\ell=1}^{n^2} e^{s_\ell}} \\
        \geq & ~ \frac{e^{C(1-\epsilon)}}{e^{C(1-\epsilon)} + n^2 \cdot e^{C\epsilon}} := \Delta
    \end{align*}
    where the first step follows from basic inequality, the second step follows from the definition of $\xi_i$, the third step follows from $s_i \geq C(1-\epsilon)$. 

    For Bob to distinguish two cases, we need to compare the upper bound of case one with the lower bound of case two. 

    Since both $\delta$ and $\Delta$ have the same denominator, if the final goal is to show $\delta < \Delta$, it is sufficient to show 
    \begin{align*}
        n^2 \cdot e^{C\epsilon} < e^{C(1-\epsilon)}
    \end{align*}
    The above equation is equivalent to 
    \begin{align*}
        n^2 < e^{C(1-2\epsilon)}
    \end{align*}
    which is further equivalent to 
    \begin{align}\label{eq:app_c_inequality_square}
         C &> \frac{2\ln n}{1-2\epsilon}
    \end{align}

    When $n \geq 3$ and $\epsilon \in (0,0.1)$, we choose $C = 4 \ln n$, which we automatically satisfy Eq.~\eqref{eq:app_c_inequality_square}, allowing Bob to determine the $j$-th component of the attention output correctly.
    
    Then we conclude the overall bits communicated is equal to the memory of $\mathcal{A}$, which is $\Omega(n^2d)$ bits.
    
    Thus, we complete the proof.
\end{proof}

This theorem demonstrates that precomputing the tensor structure in two cache matrices incurs a memory cost of $\Omega(n^2 d)$, a quadratic increase over the four-matrix case, highlighting the trade-off between preprocessing and space complexity in attention-based computations.

\subsection{Lower Bound on Approximation Algorithms for Four Cache Matrices Case}\label{sec:app_mem_lb:approx_alg_four}
We now extend the result to the approximate computation of the attention function with four cache matrices.
Using the similar idea as Theorem 8 in \cite{ho25}, we can show the following result.
\begin{corollary}[Approximation Algorithm Lower bound for Four Cache Matrices, Formal Version of Corollary~\ref{thm:lower_bound_appx_four}]\label{thm:app_tensor_approx_lower_bound_four_matrices}

If the following conditions hold:
\begin{itemize}
    \item Let $Z_n := \mathsf{Attn}(q_n, K_{1,n}, K_{2,n}, V_{1,n}, V_{2,n})$ as in Definition~\ref{def:attn_four_cache}.
    \item Let $d = \Omega(\log n)$.
\end{itemize}
Then we have:
\begin{itemize}
    \item Any algorithm that can, with probability at least $9/10$, produce an output $\mathcal{O} \in \R^d$ that is a $(1\pm\eta)$-approximation of $Z_n$ for $\eta \in (0,1)$ must use at least $\Omega(nd)$ bits of memory.
\end{itemize}
\end{corollary}

\subsection{Lower Bound on Approximation Algorithms for Two Cache Matrices Case}\label{sec:app_mem_lb:approx_alg_two}
Next, we consider the approximate computation of the attention function with two cache matrices. Using the similar idea as Theorem 8 in \cite{ho25}, we can show the following result.
\begin{corollary}[Approximation Algorithm Lower bound for Two Cache Matrices, Formal Version of Corollary~\ref{thm:lower_bound_appx_two}\label{thm:app_tensor_approx_lower_bound_two_matrices}]
If the following conditions hold:
\begin{itemize}
    \item Let $Z_n := \mathsf{Attn}(q_n, \wt{K}_n, \wt{V}_n)$ as in Definition~\ref{def:attn_two_cache}.
    \item Let $d = \Omega(\log n)$.
\end{itemize}
Then we have:
\begin{itemize}
    \item Any algorithm that can, with probability at least $9/10$, produce an output $\mathcal{O} \in \R^d$ that is a $(1\pm\eta)$-approximation of $Z_n$ for $\eta \in (0,1)$ must use at least $\Omega(nd)$ bits of memory.
\end{itemize}
\end{corollary}

\subsection{Lower Bound in the Low-Dimensional Regime for Four Cache Matrices Case}\label{sec:app_mem_lb:lower_bound_low_dim_four}

Using the similar idea as Corollary 9 in \cite{ho25}, we can provide the following lower bound with two cache matrices when $d = o(\log n)$.

\begin{corollary}\label{cor:app_lower_bound_low_dim_four}
Let $Z_n := \mathsf{Attn}(q_n, K_{1,n}, K_{2,n}, V_{1,n}, V_{2,n})$ as in Definition~\ref{def:attn_four_cache}, and $d \geq 2$ with $d = o(\log n)$. Any algorithm that can, with probability at least $9/10$, produce a $(1\pm\eta)$-approximation $\mathcal{O} \in \R^d$ of $Z_n$ for $\eta \in (0,1)$ must use at least $\Omega(e^d \cdot d)$ bits of memory.
\end{corollary}

This corollary establishes a fundamental limit on memory efficiency in the low-dimensional regime, showing an exponential dependence on $d$. It suggests that sublinear space complexity in $n$ may be achievable, prompting exploration of algorithms tailored to this setting.

\subsection{Lower Bound in the Low-Dimensional Regime for Two Cache Matrices Case}\label{sec:app_mem_lb:lower_bound_low_dim_two}
Using the similar idea as Corollary 9 in \cite{ho25}, we can provide the following lower bound with two cache matrices when $d = o(\log n)$.

\begin{corollary}\label{cor:app_lower_bound_low_dim_two}
Let $Z_n := \mathsf{Attn}(q_n, \wt{K}_n, \wt{V}_n)$ as in Definition~\ref{def:attn_two_cache}, and $d \geq 2$ with $d = o(\log n)$. Any algorithm that can, with probability at least $9/10$, produce a $(1\pm\eta)$-approximation $\mathcal{O} \in \R^d$ of $Z_n$ for $\eta \in (0,1)$ must use at least $\Omega(e^d \cdot d)$ bits of memory.
\end{corollary}

\subsection{Definition of Covering Number}\label{sec:app_mem_lb:covering_number_def}
To analyze clusterability in low-dimensional spaces, we introduce the covering number, a geometric concept that quantifies the complexity of covering a set with balls of a given radius. This definition, adapted from~\cite{w17}, supports our subsequent bounds.

\begin{definition}[Covering Number,~\cite{w17}]
Define the covering number $N(\Theta, \|\cdot\|, \delta)$ of a set $\Theta$ with respect to a norm $\|\cdot\|$ and a radius $\delta$ as the minimum number of $\|\cdot\|$-balls of radius $\delta$ needed to cover $\Theta$.
\end{definition}

The covering number provides a tool to estimate the number of clusters required for a dataset, linking geometric properties to algorithmic efficiency in low-dimensional settings.

\subsection{Clusterability of \textsc{SubGen4Cache} and \textsc{SubGen2Cache} Algorithm}\label{sec:app_mem_lb:subgen_clusterability}

Similarly \cite{ho25}, we use the covering number, we establish the inherent clusterability of points in low-dimensional spaces.

\begin{lemma}[Clusterability of \textsc{SubGen4Cache} and \textsc{SubGen2Cache} Algorithms, Formal Version of Lemma~\ref{lem:subgen_clusterability}]\label{lem:app_subgen_clusterability}
If the following conditions hold:
\begin{itemize}
    \item Let $x_1,\hdots,x_n \in \R^d$ be a set of $n$ points in $d$-dimensional space such that $\|x_i\|_2 \leq 1$ for all $i \in [n]$.
\end{itemize}
Then:
\begin{itemize}
    \item This set is $(m,\delta)$ clusterable for $m = \lceil e^d \rceil$ and $\delta = e/3$.
\end{itemize}
\end{lemma}

\begin{proof}
Lemma~\ref{lem:covering_number} states that we can cover the unit-ball with at most $(\frac{3}{\delta})^d = e^d$ balls of radius $\delta$. We then partition the unit ball into $m$ clusters by assigning each point to a ball that contains it.
\end{proof}